\documentclass{article}
\usepackage{microtype}
\usepackage{graphicx}
\usepackage{enumitem}
\usepackage{multirow}
\usepackage{float}
\usepackage{hyperref}
\usepackage{balance}
\usepackage{flushend}
\usepackage{algorithm} 
\usepackage{algorithmic}  
\usepackage[algo2e]{algorithm2e} 
\usepackage[utf8]{inputenc}
\usepackage[T1]{fontenc}    
\usepackage{hyperref}       
\usepackage{url}            
\usepackage{booktabs}       
\usepackage{nicefrac}       
\usepackage{microtype}      
\usepackage{amsthm}
\usepackage{amsmath}
\usepackage{amssymb}
\usepackage{amsfonts}       
\usepackage{mathtools}
\usepackage{mathrsfs}
\usepackage{xcolor}
\usepackage{relsize}    
\usepackage{graphicx}
\usepackage{enumerate}
\usepackage{pgfplots}
\usepackage{filecontents}
\usepackage{multicol}
\usepackage{tikz}
\usepackage{tikz-imagelabels}
\usepackage{environ}
\usepackage[symbol]{footmisc}
\usepackage{pgfplots}
\usetikzlibrary{3d}
\usepackage{wrapfig}
\usepackage{tikz-3dplot}
\usepgfplotslibrary{patchplots}
\usetikzlibrary{arrows}
\usepgfplotslibrary{fillbetween}
\usepackage{blindtext}

\usepackage{subcaption}
\usepackage{makecell}
\usepackage[font=scriptsize]{caption}

\usepackage{sidecap}

\newcommand{\argmin}{\mathop{\mathrm{argmin}}\limits}

\newcommand*\circled[1]{\tikz[baseline=(char.base)]{
            \node[shape=circle,draw,inner sep=0.6pt] (char) {\tiny #1};}}
\SetCommentSty{mycommfont}
\SetKwComment{Comment}{$\blacktriangleleft$\ }{$\blacksquare$}
\SetKwComment{tcc}{$\blacktriangledown$\ }{$\blacktriangledown$}
\pdfoutput=1 

\SetKwInput{KwInput}{Input}                
\SetKwInput{KwOutput}{Output}              
\SetKwInput{KwParameter}{Parameter}              

\newtheorem{thm}{Theorem}[section]
\newtheorem{cor}[thm]{Corollary}
\newtheorem{prop}{Proposition}[section]

\newtheorem{conjecture_}[thm]{Conjecture}
\newtheorem{definition_}{Definition}[section]

\newtheorem*{remark_}{Remark}
\pgfplotsset{compat=1.5} 

\newcommand{\executeiffilenewer}[3]{%
\ifnum\pdfstrcmp{\pdffilemoddate{#1}}%
{\pdffilemoddate{#2}}>0%
{\immediate\write18{#3}}\fi%
}

\usepackage[accepted]{icml2023}
\icmltitlerunning{Gated Compression Layers for Efficient Always-On Models}

\begin{document}
\twocolumn[
\icmltitle{Gated Compression Layers for Efficient Always-On Models}

\icmlsetsymbol{equal}{*}
\begin{icmlauthorlist}
\icmlauthor{Haiguang Li}{google}
\icmlauthor{Trausti Thormundsson}{google}
\icmlauthor{Ivan Poupyrev}{google}
\icmlauthor{Nicholas Gillian}{google}
\end{icmlauthorlist}

\icmlaffiliation{google}{Google LLC, Mountain View, CA 94043, USA}

\icmlcorrespondingauthor{Haiguang Li}{haiguang@google.com}
\icmlkeywords{Machine Learning, ICML}
\vskip 0.3in
]

\printAffiliationsAndNotice{}

\begin{abstract}
Mobile and embedded machine learning developers frequently have to compromise between two inferior on-device deployment strategies: sacrifice accuracy and aggressively shrink their models to run on dedicated low-power cores; or sacrifice battery by running larger models on more powerful compute cores such as neural processing units or the main application processor. In this paper, we propose a novel \emph{Gated Compression} layer that can be applied to transform existing neural network architectures into Gated Neural Networks. Gated Neural Networks have multiple properties that excel for on-device use cases that help significantly reduce power, boost accuracy, and take advantage of heterogeneous compute cores. We provide results across five public image and audio datasets that demonstrate the proposed Gated Compression layer effectively stops up to 96\% of negative samples, compresses 97\% of positive samples, while maintaining or improving model accuracy.
\vspace*{-8pt}
\end{abstract}
\section{Introduction}
\setlength{\tabcolsep}{9pt}
Advancements in lightweight architectures \citep{tan2019efficientnet}, on-device libraries \citep{david2021tensorflow}, and dedicated hardware accelerators have resulted in the ubiquitous deployment of machine learning models across millions of mobile, wearable, and smart devices. These advancements are powering a rapidly expanding category of on-device use-cases in \textbf{Always-On computing}. Always-On models are deployed today across millions of mobile devices, smart watches, fitness trackers, earbuds, smart doorbells, and beyond to enable use cases as broad as speaker detection; user authentication; activity recognition; noise reduction; fall detection; music classification; fault prevention; earthquake prediction; accident alerting; and more.

\begin{figure}[t!]
    \centering
    \includegraphics[width=0.48\textwidth]{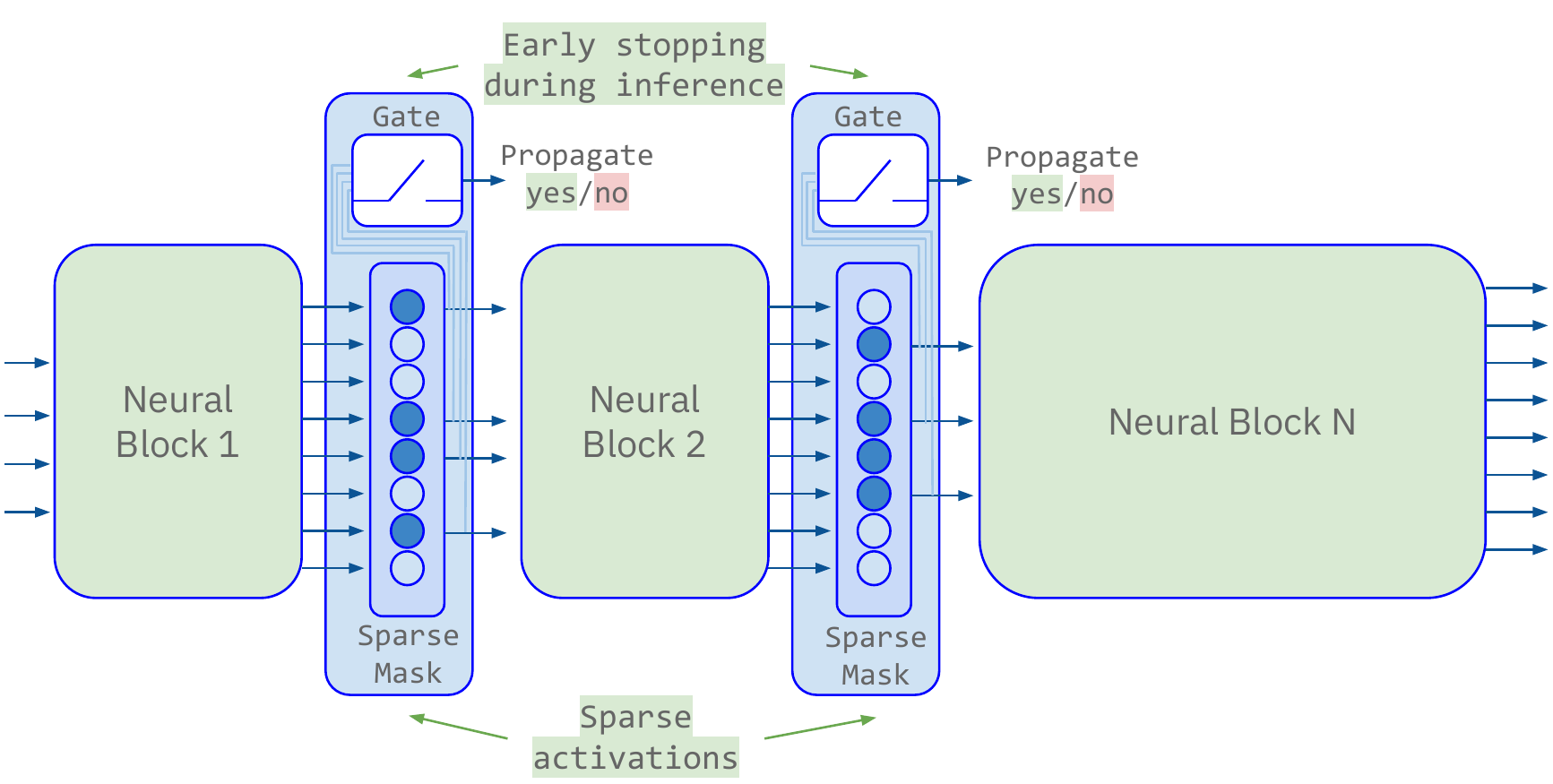}\vspace*{-10pt}
    \caption{The proposed Gated Compression layer. One or more layers can be added to existing architectures to transform any network into an efficient Always-On Gated Neural Network.}
    \label{fig:proposed_method}\vspace*{-12pt}
\end{figure}

Always-on models run continually, searching for potential signals of interest in a continuous stream of unsegmented sensor data.  We refer to signals of interest as \emph{positive} samples, with generic background data containing no signal of interest as \emph{negative} samples. In real-world data, positive samples can be sporadic, hidden in an overwhelming negative data stream. A user might speak a few keywords a day; perform an activity multiple times per week; while the majority of users ideally never experience rare events like serious accidents. This exposes the main challenge with Always-On models - they are \emph{always on}, continually searching for potential events; yet the events are \emph{sparse}.

While modern mobile devices and wearables contain dedicated heterogeneous hardware to support running lightweight models at low power, its not enough to support the orders of magnitude increase in diversity and complexity of future use cases. Consumers desire more extensive and helpful experiences from their devices, while simultaneously expecting longer battery life and reduced climate impact. To fundamentally address this problem we need to think differently. We need machine learning techniques that enable us to move from an Always-On paradigm, to context aware models that only run when needed. Moreover, as model size continues to grow, we need options that enable larger context-aware models to be efficiently distributed across the \emph{multiple} heterogeneous computing cores that are available in today's modern devices (e.g. Always-On accelerators, DSPs, Neural cores, ...). This enables much larger models to be run that can fit on any one processor, with the front-end of the model to run on extremely low-power Always-On accelerators to efficiently detect potential signals of interest and later stages of the model running on more compute-intensive processors only being triggered when appropriate.

\begin{figure*}[t]
	\centering\includegraphics[width=1.01\textwidth]{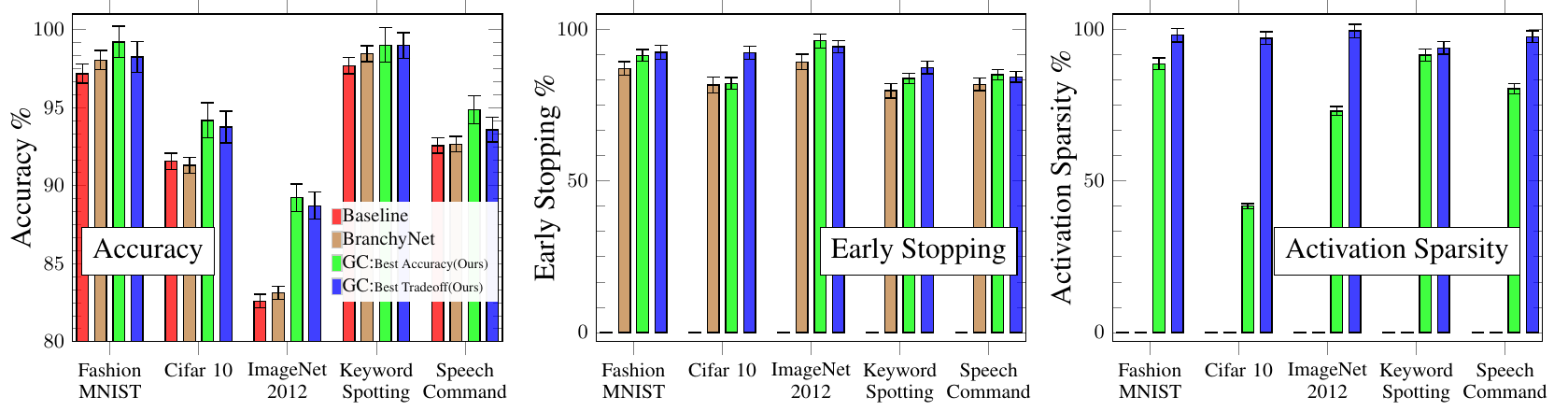}
	\caption{The proposed GC layer and two baseline techniques applied to five image and audio benchmark datasets. The GC $\alpha$, $\beta$ parameters enable models to be tuned to tradeoff accuracy vs early stopping vs compression based on the target use case requirements.}
	\label{fig:overall_plot}
	\vspace{-8pt}
\end{figure*}

In this paper, we present a novel \emph{\underline{G}ated \underline{C}ompression} (GC) layer that can be applied to existing deep neural network architectures to transform any standard network into a \textbf{Gated Neural Network}. GC layers have the following important properties:

\begin{itemize}[noitemsep,topsep=0pt,leftmargin=*]
\item \textbf{Early Stopping}: GC layers provide early stopping during \textit{inference}. GC layers can be strategically placed throughout a network to immediately stop propagation of data when there is no signal of interest, significantly reducing compute and power. GC layers are jointly optimized during training to maximize early stopping, while not degrading accuracy or other key metrics;
\item \textbf{Activation Sparsity}: GC layers automatically learn to compress intermediate feature data via activation sparsity to minimize the positive data that is propagated through an active network. Reducing feature dimensionality is key in heterogeneous computing systems due to the high cost of data transfer;
\item \textbf{Distributed Models}: GC layers provide natural delimiters to enable large-scale models to be split and distributed over multiple compute islands on a single device, or potentially multiple devices and the cloud - while still running at low power due to early stopping and layer compression;
\item \textbf{Holistic Optimization}: GC layers allow joint optimization of true positive detections while minimizing false positives errors. Multiple GC layers can be added to a single network if needed, which we show increases true positive detections while suppressing false positives.
\end{itemize}

The contributions of this paper are as follows:
\begin{itemize}[noitemsep,topsep=0pt,leftmargin=*]
\item We propose a new \emph{Gated Compression layer} and present an \emph{effective loss function that enables on-device models to be explicitly fine-tuned to specify the importance of early stopping, activation sparsity to reduce data transfer through the model, and boost overall model performance};
\item We show how \emph{combing the gating and compression components in one layer improves all metrics over independent gating or compression}.
\item Furthermore, we show the performance impact of both the position and number of GC layers with a number of common model architectures finding that \emph{gating and compression can be improved with multiple gates cascaded through a network};
\item We demonstrate through extensive experiments across 5 public image and audio datasets that various deep neural network models can be extended with \emph{Gated Compression layers to achieve up to 96\% early stopping while also boosting accuracy} (see Figure \ref{fig:overall_plot}).
\end{itemize}

\section{GC Layers for Always-On Models}

For Always-On use cases involving ML models on low-power compute cores, by reducing the data transmission and computation needs, we can improve power efficiency, battery life and resource utilization while maintaining or improving accuracy. 
We present three core ideas in this paper to improve the efficiency of Always-On models:
(i) \textit{Gated/Early Stopping}: minimizing the number of negative samples propagated through a network by adding gates to stop unnecessary data transmission and computation;
(ii) \textit{Activation Compression}: minimizing the amount of positive samples propagated through a network at key bottleneck layers; (iii) \textit{Distributed Models}: distributing a larger high-performance model over multiple heterogeneous compute stages across one or even multiple devices.

We first outline the framework to facilitate the introduction of the core ideas.
A deep neural network, which consists of a chain of layers that are processed sequentially, can be split into smaller networks.
The same results can be obtained by invoking the smaller networks sequentially. 
More specifically, a network $\mathcal{F}$ can be split into $k$ disjoint smaller networks 
$  \mathcal{F}=\{\mathcal{F}^{i}\}_{i=1}^{k},$
such that $\mathcal{F}^{i}$ consumes the output of $\mathcal{F}^{i-1}$ and produces the input for $\mathcal{F}^{i+1}$. The input $x_i$ and output $y_i$ of the $i$-th smaller network $\mathcal{F}^{i}$ can be written as:
\begin{equation*}
\begin{matrix}
x_{i} &=&\left\{\begin{matrix}
    x & i=1\\ 
    \mathcal{F}^{1\mapsto i-1}\left(x\right) & i>1
    \end{matrix}\right.\\[6pt]
y_{i} &=&\mathcal{F}^{i}\left(x_i\right) \hspace{8pt}= \hspace{8pt}\mathcal{F}^{1 \mapsto i}\left(x\right)
\end{matrix},
\end{equation*}
\newpage
\begin{equation*}
\text{where\ \ \ } \mathcal{F}^{i \mapsto j}\left(\cdot\right)=
 \mathcal{F}^{j}\left( \mathcal{F}^{j-1}\left(\cdots \mathcal{F}^{i}\left(\cdot\right) \cdots \right )\right).
\end{equation*}

\subsection{Gated/Early Stopping}\label{gate_subsection}

In Always-On use cases, where the data stream is often dominated by negative samples, it is more efficient to early stop the transmission and computation of negative samples rather than processing them end-to-end. 
By early stopping, the data transmission and computations on later smaller networks can be skipped without degrading the performance.

Similar to the branch exit in BranchyNet \citep{teerapittayanon2016branchynet}, a gate $\mathcal{G}^{i}$ is designed to stop data transmission and computation in any subsequently smaller network $\mathcal{F}^{j},\forall j\in(i,k]$. 
The gate $\mathcal{G}^{i}$, which is a binary gate, can be trained together with $\{\mathcal{F}^{m}\}_{m=1}^i$ to minimize the gate loss:
\begin{equation}\label{gate_loss}
    \mathcal{L}_\text{gate}^i(x,y)=\mathcal{L}\left(\mathcal{G}^{i}\left(\mathcal{F}^{1\mapsto i}\left(x\right)\right), \Omega^{i}(y)\right),
\end{equation}
where $\mathcal{L}(\cdot,\cdot)$ is a loss function (e.g. cross entropy), and $\Omega^{i}(\cdot)$ is a  predefined, problem-specific class mapping function for determining interesting classes.

\subsection{Activation Compression}

In Always-On use cases, positive samples should always propagate through the network end-to-end. Two adjacent smaller networks are connected together and the internal feature/activation maps are transmitting through them. 
The data amount can be substantial and the data transmission may cross boundaries, such as processors or devices.
Thus, the data transmission can consume a significant amount of power, especially in Always-On scenarios. 

To reduce the amount of transmitted data for positive samples, this paper proposes a \textit{Compression} layer $\mathcal{C}_\varphi:\mathbb{R}^n\mapsto\mathbb{R}^n$ to learn data compression:
\begin{equation}\label{compression_droput_layer}
    y=x\circ \varphi ,
\end{equation}

where $x$, $\varphi$, and $y$ are the input, weight, and output, respectively. The notation ‘$\circ$’ represents element-wise product. As a linear transformation, adding this type of layer, the entire network can still be trained using end-to-end back-propagation. 

A \textit{Compression} layer $\mathcal{C}^{i}$, which acts as a bottleneck and a bridge between $\mathcal{F}^{i}$ and $\mathcal{F}^{i+1}$, can be trained together with $\{\mathcal{F}^{m}\}_{m=1}^k$ to minimize the compression loss:

\begin{align}\label{compression_loss}
\begin{split}
    \mathcal{L}_\text{comp}^i(x,y) = &\mathcal{L}\left(\mathcal{F}^{i+1\mapsto k}\left(\mathcal{C}^{i}\mathcal{F}^{1\mapsto i}\left(x\right)\right), y\right)+\\& \beta\mathcal{L}_\text{trans}(\varphi),
\end{split}\raisetag{16pt}
\end{align}
where $\mathcal{L}_\text{trans}(\cdot)$ is a sparsity regularization term (e.g. $\mathcal{L}_1$) applied to the weight matrix $\varphi$ to promote sparsity in the activation outputs.
Thus, the compression rate is controlled by the hyperparameter $\beta$, which enables fine-grained control over the compression.

Inspired by the binarized neural networks \citep{courbariaux2016binarized}, the weight matrix of the \textit{Compression} layer is binarized to reduce its parameter size for deploying.
The use of binary weights provides a natural way to represent which dimensions should be compressed (0) and which should be passed through (1).

\subsection{GC Layer}

A \textit{Gate} stops negative samples early, while a \textit{Compression} layer minimizes the amount of transmitted data for positive samples. It is intuitive to combine them into a signal layer, and insert into an existing network to gain the benefit of both. Additionally, our experimental results show this combination is practically beneficial, as it improves the performance of each other further.

A GC layer, a combination of the \textit{Gate} and \textit{Compression} layers, is proposed to minimize the amount of transmitted data and computation required.
Here, we assume that the \textit{Gate} $\mathcal{G}$ takes the output of the \textit{Compression} layer $\mathcal{C}$ as input. Theoretically, it can take other inputs as well. 
However, according to our experimental results, it is advantageous to take the output of $\mathcal{C}$: (a) the connection is within the GC layer, making the GC layer easier to integrate into an existing network; (b) the input is sparse, resulting in a smaller $\mathcal{G}$, (c) $\mathcal{G}$ performs better as all layers before it can be fine-tuned for better representation learning.

A GC layer, $\mathcal{GC}^{i}=\{\mathcal{G}^{i}, \mathcal{C}^{i}\}$, which acts as both a bridge and a gate between $\mathcal{F}^{i}$ and $\mathcal{F}^{i+1}$. There are various strategies to train $\mathcal{G}^{i}$ and $\mathcal{C}^{i}$, but our experiments indicate that training them together end-to-end with the entire network yields the best performance. Therefore, the GC layer is trained together with $\{\mathcal{F}^{m}\}_{m=1}^k$ for minimizing the $\mathcal{L}_\mathcal{GC}^i$ loss:

\begin{align*}
\begin{split}
    \mathcal{L}_\mathcal{GC}^i(x,y)=&\alpha\underbrace{\mathcal{L}\left(\mathcal{G}^{i}\left(\mathcal{C}^{i}\mathcal{F}^{1\mapsto i}\left(x\right)\right), \Omega^{i}(y)\right)}_\textup{\tiny\circled{1}Gate Loss}
    + \\ &\beta\underbrace{\mathcal{L}_\text{trans}(\varphi)}_\textup{\tiny\circled{2}Trans. Cost}
    +\\ &\eta\underbrace{\mathcal{L}\left(\mathcal{F}^{i+1\mapsto k}\left(\mathcal{C}^{i}\mathcal{F}^{1\mapsto i}\left(x\right)\right), y\right)}_\textup{\tiny\circled{3}Final Prediction Loss} 
\end{split},
\end{align*}

where $\alpha_i$ and $\beta_i$ control the early stopping and compression performance of the GC layer, respectively. The wieght $\eta$ is for the final prediction loss.

\textit{Implicit Pre-Training and Feature Selection.}
Besides enhancing power efficiency and resource utilization, our experiments also show that GC layers improve accuracy. This is due to the following:
(1) The \textit{Gate} guides the early layers towards a more favorable direction, similar to \textit{pre-training}, and enables for early stopping of negative samples, allowing later layers to focus on positive samples; (2) The \textit{Compression} layer discards irrelevant or partially relevant dimensions, similar to \textit{feature selection}; (3) The GC layer, which combines the \textit{Gate} and \textit{Compression} layers, enhances the network efficiency and regularizes it to prevent overfitting.

\subsection{Objective Function}

Let $(x, y)$ be a data pair drawn from a distribution $\mathbf{P}(\mathbb{X},\mathbb{Y})$, where $x\in\mathbb{X}$ is a sample in $\mathbb{R}^{n}$ and $y \in\mathbb{Y}$ is the label in $\mathbb{N}$. Given a set of data pairs $\mathbf{D}(\mathbb{X},\mathbb{Y})=\{(x_i,y_i)|(x_i,y_i)\sim\mathbf{P}(\mathbb{X},\mathbb{Y})\}_{i=0}^N$, the goal is to learn the parameters $\theta\in\Theta$ of a deep neural network, $\mathcal{F}_{\theta}: \mathbb{X}\mapsto\mathbb{Y}$, that predicts the label $y$ for $\forall x\in\mathbf{D}(\mathbb{X})$ by solving:

\begin{equation*} \label{cliassifcation}
\theta^{*}=\argmin_{\theta\in\Theta} \underbrace{\underset{\mathsmaller{\mathsmaller{(x,y)\in \mathbf{D}}}(\mathbb{X},\mathbb{Y})}{\mathlarger{\mathlarger{\mathlarger{\mathlarger{\mathlarger{\mathlarger{\mathbb{E}}}}}}}} \mathcal{L}\left(\mathcal{F}_{\theta}(x), y\right)}_\textup{\tiny Prediction loss} + \xi\underbrace{\gamma(\theta)}_\textup{\tiny Penalty},
\end{equation*}

where $\gamma(\cdot)$ is the penalty term aiming at controlling the size and structure of the network $\mathcal{F}_\theta$. The weight $\xi$ controls the strength of the penalty term.

An existing network can be split into a set of smaller networks, 
$\widetilde{\mathcal{F}}_{\psi}
    =\left\{\mathcal{F}^{i},\mathcal{GC}^{i}\right\}_{i=1}^{k-1}\cup\left\{\mathcal{F}^{k}\right\}$
, by adding GC layers.
The new network can be learned by solving:

\begin{align}\label{new_obj}
\begin{split}
\psi^{*}=
	&\argmin_{\psi\in\Psi}\underset{\mathsmaller{\mathsmaller{(x,y)\in \mathbf{D}}}(\mathbb{X},\mathbb{Y})}{\mathlarger{\mathlarger{\mathlarger{\mathlarger{\mathlarger{\mathlarger{\mathbb{E}}}}}}}}
\sum\limits_{i=1}^{k-1}\begin{bmatrix}\alpha_i\underbrace{\mathcal{L}_\text{gate}^i(x,y)}_\textup{\tiny\circled{1}Gate Loss}+\\\beta_i\underbrace{\mathcal{L}_\text{\tiny trans}^{i}({\varphi}_{\mathcal{C}}^{i})}_\textup{\tiny\circled{2}Trans. Cost}\end{bmatrix}
	 +\\&\eta\underbrace{\mathcal{L}\left(\widetilde{\mathcal{F}}^{1\mapsto k}(x), y\right)}_\textup{\tiny\circled{3}Final Prediction Loss}+\xi\underbrace{\gamma(\psi)}_\textup{\tiny\circled{4}Penalty},
\end{split}\raisetag{28pt}
\end{align}

\begin{align*}
\hspace{-12pt}\text{where \ \ \ }\mathcal{L}_\text{gate}^i(x,y) &= \mathcal{L}\left(\mathcal{G}^{i}\left(\mathcal{C}^{i}\mathcal{F}^{1\mapsto i}\left(x\right)\right), \Omega^{i}(y)\right), \\
\widetilde{\mathcal{F}}^{i\mapsto j}(x) &= \mathcal{C}^{j}\left(\mathcal{F}^j\widetilde{\mathcal{F}}^{i\mapsto j-1}(x)\right),\\
\eta &= 1-\sum_{i\in[1,k)}\alpha_{i}. 
\end{align*}

The `Gate Loss' and `Final Prediction Loss' are derived from classification error. To ensure balance in the classification loss, the weights ($\alpha_i, \forall i \in [1,k)$ and $\eta$) are normalized so that their sum is equal to 1. The `Trans. Cost' and `Penalty' are regularization terms for the network structure, there are no strict constraints on their weights ($\beta_i$ and $\xi$).

The $\mathcal{L}_\text{gate}$ terms stop negative samples early, the $\mathcal{L}_\text{trans}$ terms minimize the amount of transmitted data for positive samples, and the `Final Prediction Loss' term ensures the overall model performance. By training on all these terms together, the new network $\widetilde{\mathcal{F}}_{\psi}$ is optimized for better model performance and better power efficiency, battery life and resource utilization. 

\subsection{Distributed Model with GC Layers}

For Always-On use cases, normally there are multiple heterogeneous compute islands available. For example, a sensor is connected to a microcontroller, then a sensor hub, a mobile device, and finally even the cloud.

GC layers can spit an existing network into smaller networks. This enables the execution of $\mathcal{F}^{i}$ on the $i$-th compute island, which in turn allows for the full utilization of all available resources. As a result, a larger and more powerful network can be built for better performance.

In the distributed scenario, two adjacent smaller networks $\mathcal{F}^{i}$ and $\mathcal{F}^{i + 1}$ are running on different compute islands, and $\mathcal{F}^{i+1}$ takes the output of $\mathcal{F}^{i}$ as its input: 
$x_{i+1}\xleftarrow[\textsf{\tiny hardware boundaries}]{\textsf{\tiny Transmit across}} y_i.$
Since this data transmission crosses physical boundaries (e.g., device-to-device communication over Bluetooth or WiFi), it consumes a significant amount of power. This power consumption is amplified significantly for always-on use cases.

GC layers can create early exits and bottlenecks, which reduce the amount of data that needs to be transmitted across boundaries and decrease power consumption.

\begin{table*}[tb]
\centering
\caption{\small Detailed Results of the Proposed GC Layer on Various Datasets.}\vspace{6pt}
\label{tab:overall}
\resizebox{0.995\textwidth}{!}{
\begin{tabular}{l||llllllll}
\hline\hline
& Dataset & Architecture & Method & $\alpha$ & $\beta$  & Accuracy& {Early Stopping} & {Activation Sparsity} \\\hline\hline
\multirow{15}{*}{\shortstack[c]{Image}} & \multirow{5}{*}{\shortstack[c]{Fashion\\MNIST}} & \multirow{5}{*}{\shortstack[c]{ResNet}} &Baseline          & $\times$    & $\times$ & $0.97183\pm0.00145$ & $0\%$ & $0\%$\\
&&& BranchyNet & $\times$ & $\times$ & $0.98051\pm0.00623$& $87.03\%$ & $0\%$ \\
&&& With $\mathcal{G}$ only & $0.50$ & $\times$ & $0.98108\pm0.00182$& $93.36\%$ & $0\%$ \\
&&& With $\mathcal{C}$ only & $\times$    & $0.55$     & $0.97998\pm0.01406$& $0\%$ & $98.07\%$  \\
&&& With $\mathcal{GC}$:{\tiny Best Trade-off}    & $0.50$ & $0.55$  & $0.98255\pm0.00160$   & $92.35\%$ & $98.02\%$  \\
&&& With $\mathcal{GC}$:{\tiny Best Accuracy}     & $0.05$ & $0.10$  & $\boldsymbol{0.99219\pm0.00091}$   & $91.37\%$ & $88.61\%$  \\\cline{2-9}
&\multirow{5}{*}{Cifar10} & \multirow{5}{*}{ResNet} & Baseline           & $\times$    & $\times$ & $0.90235\pm0.00920$  & $0\%$ & $0\%$ \\
&&& BranchyNet & $\times$ & $\times$ & $0.91312\pm0.00508$& $81.57\%$ & $0\%$ \\
&&& With $\mathcal{G}$ only & $0.55$ & $\times$ & $0.92254\pm0.00216$  & $87.5\%$ & $0\%$ \\
&&& With $\mathcal{C}$ only & $\times$    & $0.05$     & $0.93962\pm0.00586$  &$0\%$ & $42.10\%$  \\
&&& With $\mathcal{GC}$:{\tiny Best Trade-off}     & $0.70$ & $0.8$     & $0.93766\pm0.06124$  &$92.23\%$ & $97.07\%$ \\
&&& With $\mathcal{GC}$:{\tiny Best Accuracy}      & $0.60$ & $0.15$     & $\boldsymbol{0.94187\pm0.00244}$  &$82.17\%$ & $41.64\%$\\\cline{2-9}

&\multirow{5}{*}{\shortstack[c]{ImageNet\\2012}} &\multirow{5}{*}{\shortstack[c]{ResNeXt}} &Baseline           & $\times$    & $\times$ & $0.82591\pm0.01085$ & $0\%$ & $0\%$ \\
&&& BranchyNet & $\times$ & $\times$ & $0.83130\pm0.00434$& $89.23\%$ & $0\%$ \\
&&&With $\mathcal{G}$ only & $0.25$ & $\times$ & $0.84316\pm0.01262$ & $95.66\%$ & $0\%$ \\
&&&With $\mathcal{C}$ only & $\times$    & $0.35$     & $0.84095\pm0.02134$ & $0\%$ & $97.84\%$ \\
&&&With $\mathcal{GC}$:{\tiny Best Trade-off}      & $0.35$ & $0.15$  & ${0.88715\pm0.02100}$  & $94.26\%$ & $94.38\%$    \\
&&&With $\mathcal{GC}$:{\tiny Best Accuracy}     & $0.45$ & $0.05$     & $\boldsymbol{0.89221\pm0.01844}$ & $96.17\%$ & $73.03\%$ \\\hline\hline

\multirow{10}{*}{\shortstack[c]{Audio}} &\multirow{5}{*}{\shortstack[c]{Keyword\\Spotting}} &\multirow{5}{*}{\shortstack[c]{ResNet}} &
Baseline           & $\times$    & $\times$ & $0.97695\pm0.00218$ & $0\%$ & $0\%$\\
&&& BranchyNet & $\times$ & $\times$ & $0.98462\pm0.00509$& $79.73\%$ & $0\%$ \\
&&&With $\mathcal{G}$ only & $0.45$ & $\times$ & $0.98875\pm0.00124$ & $91.56\%$ & $0\%$  \\
&&&With $\mathcal{C}$ only & $\times$    & $0.20$     & $0.98742\pm0.00278$ & $0\%$ & $97.77\%$  \\
&&&With $\mathcal{GC}$:{\tiny Best Trade-off}     & $0.45$ & $0.20$     & $0.98988\pm0.00133$ & $87.36\%$ & $93.84\%$  \\
&&&With $\mathcal{GC}$:{\tiny Best Accuracy}      & $0.30$ & $0.10$     & $\boldsymbol{0.99018\pm0.00164}$ & $83.72\%$ & $91.46\%$  \\\cline{2-9}
&\multirow{5}{*}{\shortstack[c]{Speech\\Command}} &\multirow{5}{*}{\shortstack[c]{Inception}} &
Baseline           & $\times$    & $\times$ & $0.92527\pm0.00876$ & $0\%$ & $0\%$\\
&&& BranchyNet & $\times$ & $\times$ & $0.92662\pm0.00489$& $81.83\%$ & $0\%$ \\
&&&With $\mathcal{G}$ only & $0.05$ & $\times$ & $0.94111\pm0.00932$  & $94.26\%$ & $0\%$ \\
&&&With $\mathcal{C}$ only & $\times$    & $0.35$     & $0.93164\pm0.00394$  & $0\%$ & $99.03\%$ \\
&&&With $\mathcal{GC}$:{\tiny Best Trade-off}     & $0.05$ & $0.35$     & $0.93590\pm0.00411$  & $84.31\%$ & $97.53\%$ \\
&&&With $\mathcal{GC}$:{\tiny Best Accuracy}      & $0.25$ & $0.10$     & $\boldsymbol{0.94879\pm0.00181}$  & $85.03\%$ & $80.36\%$ \\\hline\hline

\end{tabular}}
\end{table*}
Overall, GC layers can be used to optimize existing networks for Always-On use cases, by reducing power consumption, boosting accuracy, and utilizing heterogeneous compute islands. These optimizations are obtained by creating early exits and bottlenecks, which minimize the amount of computation and data transfer required, allowing the network to run more efficiently in Always-On scenarios.

\section{Experiments}

We apply GC layers to Always-On scenarios across both image and audio classification tasks. In this section, we first describe the datasets, evaluation protocols, and implementation details used to train and test each model, then discuss results. We then ablate key components of the GC layer and discuss key insights.

\subsection{Datasets, Architectures, and Evaluation Protocols}

Despite the ubiquitous nature of Always-On computing in today's consumer devices, there are limited public datasets that represent the true distribution of positive vs negative samples found in real-world use cases. Thankfully, common machine learning datasets such as ImageNet 2012 \citep{ILSVRC15} can be transformed into Always-On benchmark dataset by mapping a subset of labels to a generic negative class as a proxy for real-world background data.

For our experiments, we selected three common public image datasets: Fashion MNIST \citep{fashion_mnist}, Cifar10 \citep{cifar10}, and ImageNet 2012 \citep{ILSVRC15}; and two public audio datasets: Keyword Spotting V2 \citep{leroy2019federated} and Speech Command \citep{speechcommandsv2}. For each image dataset, we map every-other class to a generic background class to reflect an Always-On use case. All examples with even class labels are kept as positive examples, with all examples with odd class labels being mapped to a generic negative class. For example, in Cifar10 the even classes (airplane, bird, cat, deer, frog, ship) are unmodified while the odd classes (automobile, cat, dog, horse, truck) are grouped into a generic negative class. This simulates a 1:1 ratio between positive and negative samples. The audio datasets contain a pre-existing background class with a 1:9 ratio between positive and negative samples and required no additional class-label remapping. Note that real-world use cases can have a significant imbalance weighted towards negative samples, which only amplifies the need for techniques that support early stopping in Always-On models.

We use the following reference architectures for each dataset to demonstrate that GC layers can be applied to common model architectures (Figure \ref{fig:model_arch}): ResNets \citep{resnet2016} (Fashion MNIST, Cifar10, Keyword Spotting); ResNeXt \citep{xie2017aggregated_resnext} (ImageNet 2012); and Inception \citep{inception2015} (Speech Command).

We evaluate the accuracy vs early-stopping performance of architectures expanded with GC layers. Early stopping is defined as the percentage of negative test examples that are successfully gated by the model without having to propagate to the final classification layer.

We compare architectures expanded with GC layers against two baseline architectures: a baseline architecture with no gating and a baseline architecture that employs the popular BranchyNet gating technique   \citep{teerapittayanon2016branchynet}. For each dataset, the baseline architecture, BranchyNet architecture, and GC architecture are identical with the exception of the additional BranchyNet or GC layer. We first report the results of a single gate placed at 40\% within the depth of each network on the various image and audio test datasets. We then explore the impact of the position and number of gates, along with $\alpha, \beta$ tuning to achieve the best trade-off vs best-accuracy in the ablation studies.

\begin{figure}[t!]
	\centering\includegraphics[width=0.485\textwidth]{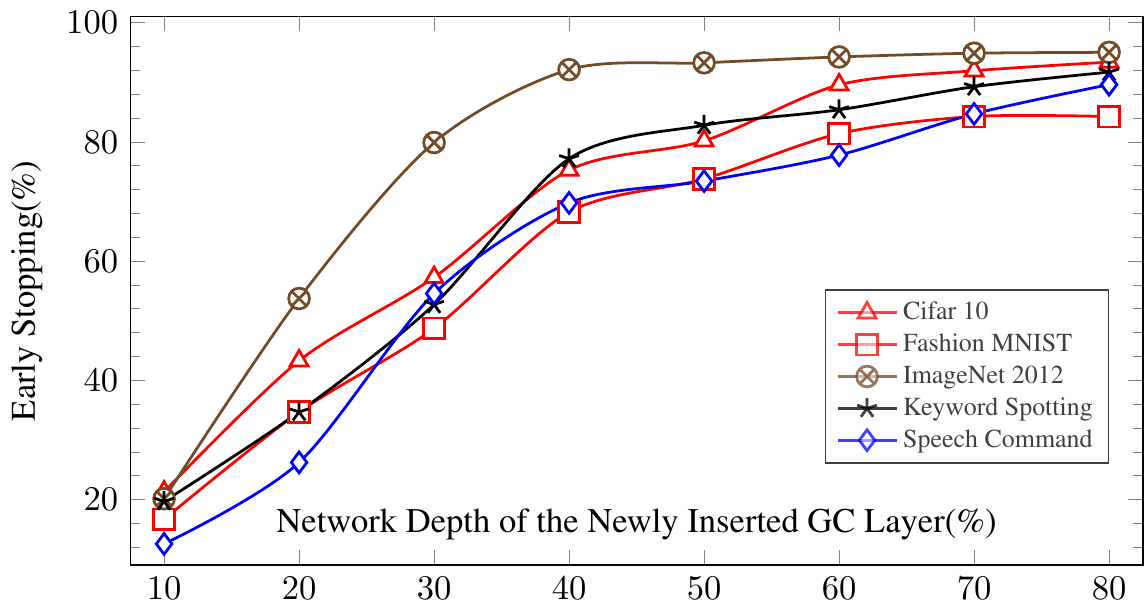}\caption{The Effect of Changing the Network Depth of A Newly Added GC Layer on Early Stopping Performance. 
	}\label{fig:early_stopping}
\end{figure}

\subsection{Implementation Details}

All methods are implemented with TensorFlow 2.x \citep{tensorflow2015-whitepaper}. Unless otherwise specified, the batch size is set to 512 and the training epoch is set to 200; the Adam \citep{adom2015} optimizer with a fixed learning rate (0.01) is used for model training. As the larger ResNeXt-101 64x4d \citep{xie2017aggregated_resnext} was used for the ImageNet dataset, a larger batch size of 1536, training epoch of 100, and learning rate of 0.006 was used to reduce training time.

For prepossessing, the audio signals are converted into Mel-frequency cepstral coefficients, prior to input to the audio ResNet or Inception models. All experiments are repeated 10 times with the mean and variance results reported.

\subsection{Accuracy vs Early Stopping Results}\label{sec:primary_experiment}

Figure \ref{fig:overall_plot} shows the results comparing the baseline models, BranchyNet models, and GC models across the image and audio datasets. Note that  the `GC:{\tiny Best Accuracy}' models consistently achieve the highest accuracy with competitive early stopping performance. 
On the other hand, the `GC:{\tiny Best Tradeoff}' models achieve improved accuracy over both BranchyNet and the reference baseline across all datasets, ranging from an improvement of 1.06 percentage points (Fashion MNIST) to 6.12 percentage points (ImageNet) over the baseline architecture.
Furthermore, they also have significantly improved early stopping performance over BranchyNet, with a range of 82.17\% (Cifar 10) to 96.17\% (ImageNet). 
Additionally, the GC models provide additional compression on the layer activation of the GC layer, reducing the feature dimensionality of any data that is transmitted to the next stage of the model, ranging from 41.64\% (Cifar 10) to 97.53\% (Speech Command).

The results in Figure \ref{fig:overall_plot} demonstrate that our GC layer can effectively identify and stop negative samples early, while simultaneously boosting model accuracy.
\begin{figure}[t!]
	\centering\includegraphics[width=0.485\textwidth]{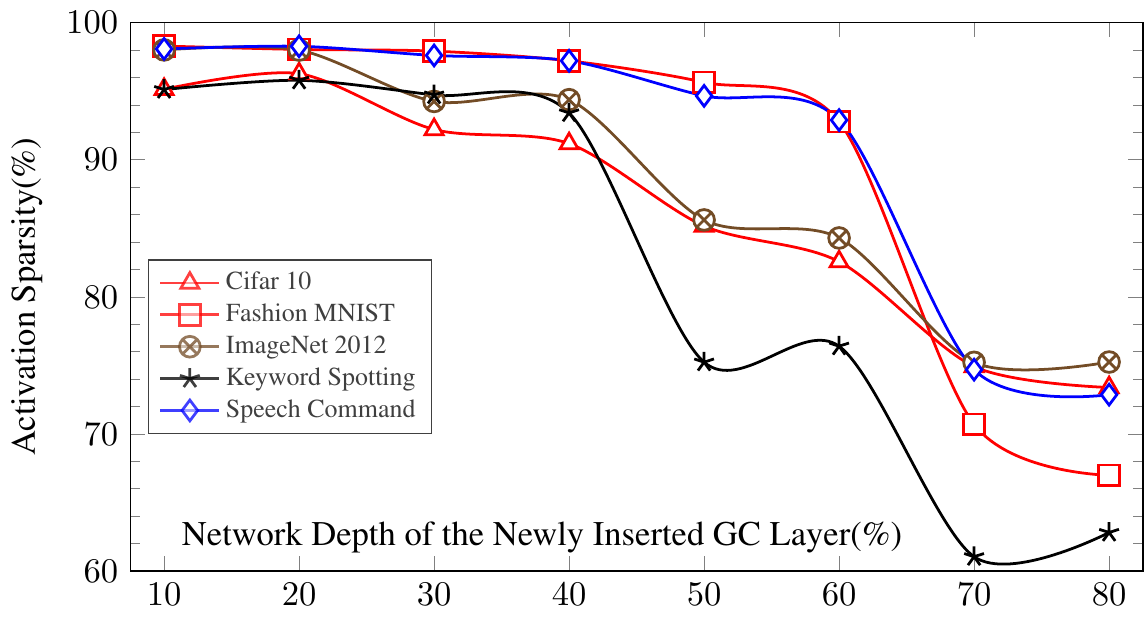}\caption{The Effect of Changing the Network Depth of A Newly Added GC Layer on Activation Sparsity.}\label{fig:compression}

\end{figure}
\begin{figure}[t!]
	\centering \centering\includegraphics[width=0.495\textwidth]{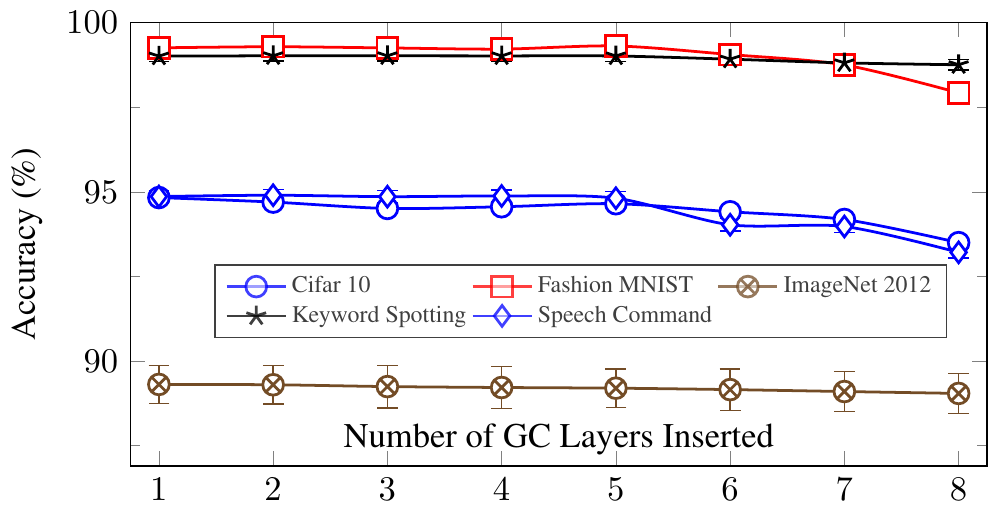}\caption{The Effect of Changing Number of GC Layers Inserted on Model Accuracy.}\label{fig:accuracy_gates}
\end{figure}
\begin{figure*}[tb!]
\captionsetup[subfigure]{labelformat=empty}

\begin{subfigure}[t]{\textwidth}
	\includegraphics[width=0.327\textwidth]{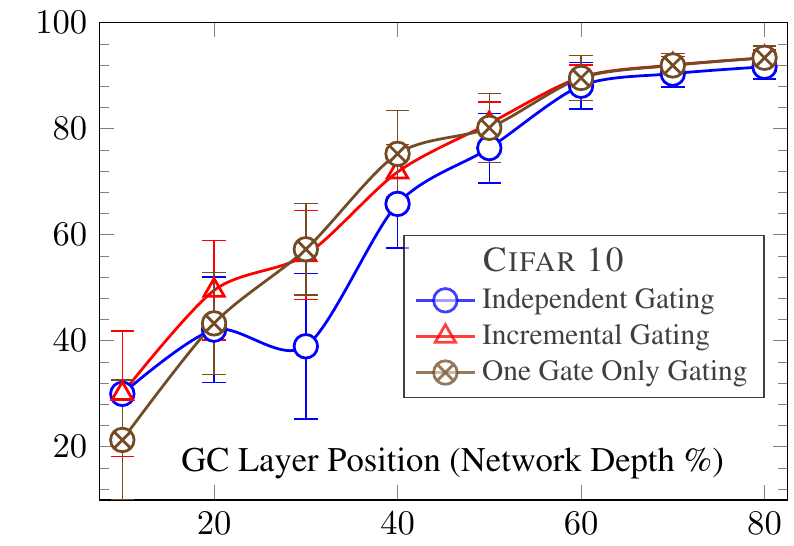}
	\includegraphics[width=0.327\textwidth]{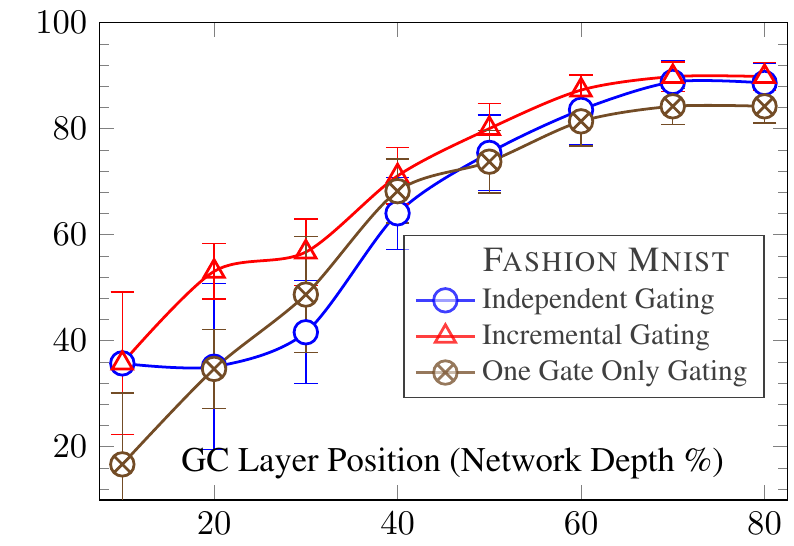}
	\includegraphics[width=0.327\textwidth]{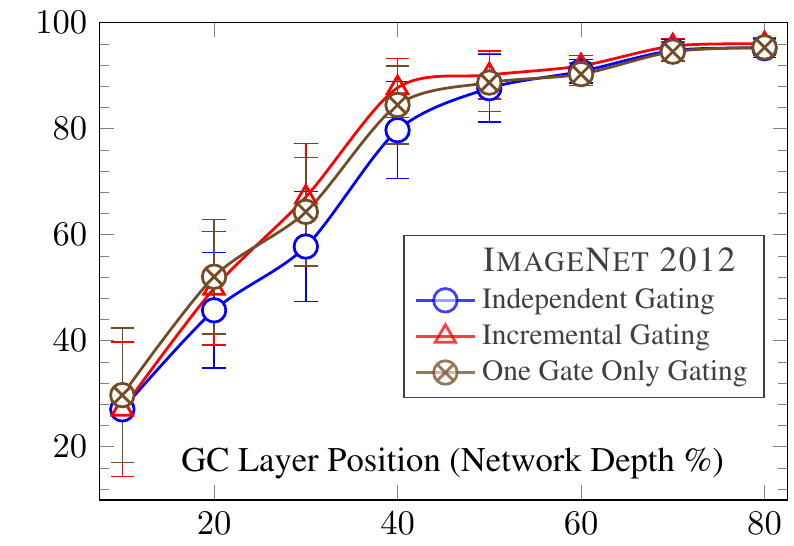}
\end{subfigure}
\begin{subfigure}[t]{\textwidth}
	\includegraphics[width=0.327\textwidth]{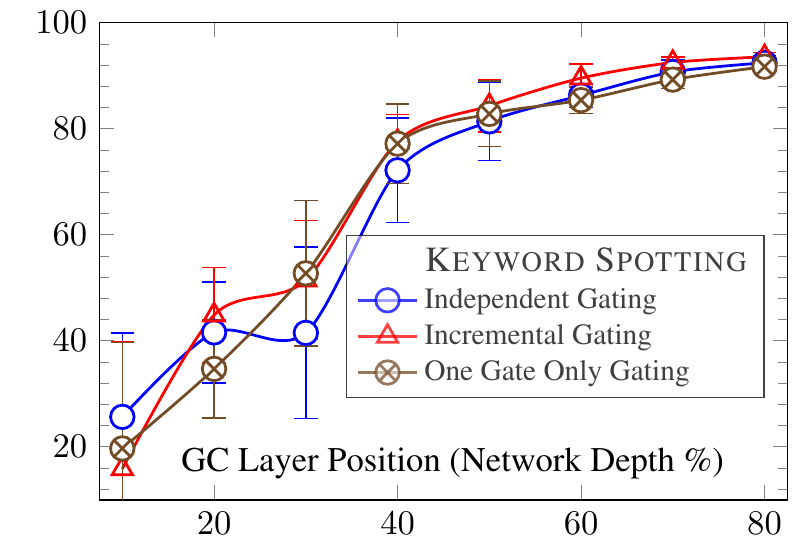}
	\includegraphics[width=0.327\textwidth]{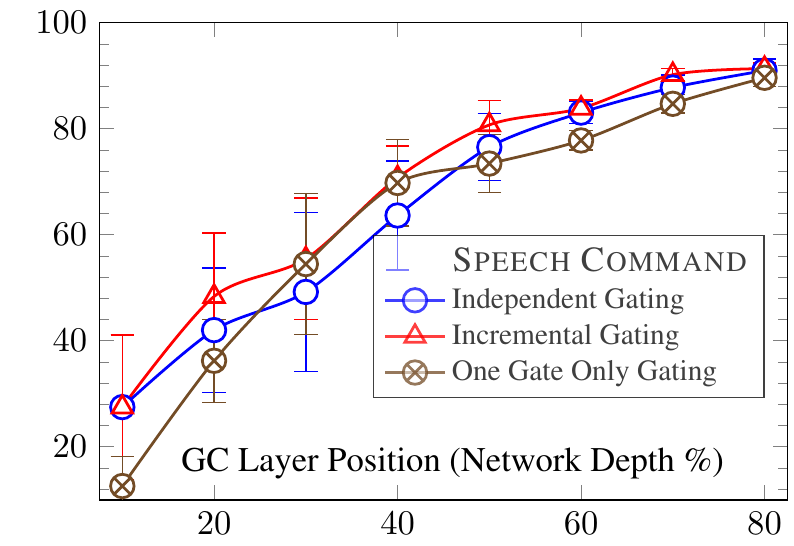}
\end{subfigure}
\caption{The Effect of Adding Multiple GC Layers on Early Stopping (\%). }\label{fig:multiple_gates_early_stopping}
\end{figure*}

\subsection{Best Trade-off vs Best Accuracy}

On-device models need to carefully balance key metrics, such as accuracy, against critical factors, such as power usage or memory constraints. To reflect this real-world prioritization, we use the GC $\alpha$ and $\beta$ parameters to train two variants of models, one weighted towards achieving the best-possible accuracy (GC:{\tiny Best Accuracy}) the second weighted towards achieving the best-possible accuracy-vs-early-stopping compromise (GC:{\tiny Best Tradeoff}).

The results in Table \ref{tab:overall} indicate that adding GC layers improves both accuracy and early stopping/compression performance. Specifically, `GC:{\tiny Best Accuracy}' models consistently achieve the highest accuracy, while `GC:{\tiny Best Trade-off}' models always obtain the best activation sparsity by balancing accuracy with more aggressive early stopping and compression. 
Specifically, the `GC:{\tiny Best Tradeoff}' models achieve activation sparsity ranging from 93.84\% (Keyword Spotting) to 97.53\% (Speech Command), while the `GC:{\tiny Best Accuracy}' models achieve activation sparsity ranging from 41.64\% (Cifar 10) to 91.46\% (Keyword Spotting).
Additionally, in either case, GC models consistently outperform the baseline and BranchyNet reference models in \emph{both} accuracy and early stopping.

Overall, the results in Table \ref{tab:overall} demonstrate that the GC layer can help improve the model accuracy further, while providing the benefits of early stopping and compression. Furthermore, depending on the use case's requirements, the GC layer can be configured with its $\alpha$ and $\beta$ parameters to prioritize early stopping and/or compression, while maintaining a high level of accuracy with a slight decrease if required. 

\subsection{Impact of GC Layer Position}

Section \ref{sec:primary_experiment} demonstrates significant early stopping and activation compression when placing a single GC layer 40\% deep within each model architecture. In this experiment, we evaluate the impact of the position of a single GC layer within the network.

Figure \ref{fig:early_stopping} shows adding a single GC layer in a network can effectively early stop negative samples, and the early stopping performance improves as the GC layer is moved to a deeper position: 
Inserting a GC layer at 10\% network depth early stops $10\sim40\%$ of negative samples, and positioning it at the 40\% depth improves early stopping to $70\sim90\%$.
This is because placing at a deeper position allows for more layers before it to be fine-tuned for better early stopping performance.

Figure \ref{fig:compression} shows that as the GC layer is moved to deeper positions, activation sparsity decreases. Placement at 40\% depth achieves a high activation sparsity of $90\sim98\%$, but at 80\% depth results in a lower activation sparsity of $60\sim78\%$.
This is because, as the GC layer is placed in shallower positions, it can compress more dimensions due to the larger internal feature map size.

Overall, Figures \ref{fig:early_stopping} and \ref{fig:compression} demonstrate that a single GC layer can provide early stopping and compression benefits when inserted at various depths. Specifically, as the GC layer is placed deeper in the network, the early stopping performance improves while the compression performance declines.
This insight can inform the positioning when inserting a GC layer based on the use case's requirements.

\subsection{Impact of Multiple GC Layers}

To evaluate the effect of increasing the number of GC layers on the models' performance, a set of experiments have been performed by inserting different numbers of GC layers into the existing baseline networks.
To distribute the GC layers evenly, a balanced approach is chosen to select their positions. For example, to insert 4 GC layers, they are placed at the 20\%, 40\%, 60\%, and 80\% network depths.

The results in Figure \ref{fig:accuracy_gates} show that 
the accuracy remains almost unchanged when the number of GC layers is less than 5. However, when the number of GC layers exceeds 5, there is a noticeable decrease in accuracy of less than 2 percentage points for the Speech Command, Fashion MNIST, and Cifar 10 datasets.


To evaluate the early stopping performance with multiple GC Layers, three methods are used for comparison: (1) Independent Gating: 8 GC layers are inserted into the baseline model for training, only the gate in the GC layer at the specified position is active during inference; (2) Incremental Gating: 8 GC layers are inserted into the baseline model for training, gates in GC layers after the specified position are disabled during inference; (3) One Gate Only Gating: Only one GC layer is inserted into the baseline model at the specified position for training and inference.

The results in Figure \ref{fig:multiple_gates_early_stopping} show that using incremental gating with multiple GC layers improves the early stopping performance compared to using `independent gating' and `one gate only gating'.

Overall, the results in Figures \ref{fig:accuracy_gates} and \ref{fig:multiple_gates_early_stopping} demonstrate that inserting multiple GC layers into a single network can provide various benefits:
(1) Improved Early Stopping: multiple GC layers can incrementally stop more negative samples;
(2) Multi-Stage Activation Compression: the amount of transmitted data can be reduced further by compressing with GC layers at different positions;
(3) Flexibility: the position and number of GC layers can be adjusted based on the use case's requirements.

\vspace{3pt}
\section{Related Work}
\vspace{3pt}
Deep neural networks have shown superior performance in many computer vision and natural language processing tasks. Recently, an emerging amount of work is applying deep neural networks on resource constrained edge devices \citep{dhar2019device}.

\textit{Model compression} is a popular approach for resource efficiency. The model size is compressed via techniques such as network pruning, vector quantization, distillation, hashing, network projection, and binarization \citep{Alperen_pruning_2022,liu2020pruning, wang2019haq, ravi2017projectionnet, courbariaux2016binarized, Distilling2015, han2015deep, chen2015compressing,  gong2014compressing}.

By reducing weights and connections, a lot of \textit{light weight architectures} were proposed for edge devices: MobileNets v1 \citep{howard2017mobilenets}, v2 \citep{sandler2018mobilenetv2} and v3 \citep{howard2019searching}, SqueezeNet \citep{iandola2016squeezenet} and SqueezeNext \citep{gholami2018squeezenext}, ShuffleNet \citep{zhang2018shufflenet}, CondenseNet \citep{huang2018condensenet}, and the NAS generated MnsaNet \citep{tan2019mnasnet}. These new lightweight architectures reduce model size and resource requirements while retaining fairly good accuracy. 

\textit{Quantization} reduces model complexity by using lower or mixed precision data representation. There are huge amount of emerging studies exploring 16-bit or lower precision for some or all of numerical values without much degradation in the model accuracy \citep{cambier2020shifted, guo2018survey,micikevicius2017mixed, judd2015reduced, wang2018training}.

Encouraging \textit{sparse structure} of the model architecture is able to reduce model complexity. The group lasso regularization \citep{feng2015learning, lebedev2016fast, wen2016learning} and learnable dropout techniques \citep{boluki2020learnable, molchanov2017variational} are efficient ways to encourage sparse structures in various deep neural network components and weights.

To utilize resources across hardware boundaries, \textit{distributed deployment techniques} have been proposed \citep{teerapittayanon2017distributed, mcmahan2017communication, teerapittayanon2016branchynet, tsianos2012consensus, ouyang2017chained,Gormez_2022,Alperen_class_2022,kaya2019shallow} to deploy deep neural network over multiple compute islands.

Our work is related to both distributed deployment and sparse structure. For better model performance, the proposed GC layer can distribute a single model across heterogeneous compute islands to fully utilize all resources available. To reduce the data transmission and computation needs, the GC layer allows early stopping for data with no signal of interest and minimizes the amount of transmitted data for other samples.

\section{Conclusion}
In this paper, we introduce a novel Gated Compression (GC) layer that can be incorporated into existing neural network architectures to convert them into Gated Neural Networks. This allows standard networks to benefit from the advantages of gating, such as improved performance and efficiency.


The GC layer is a lightweight layer for efficiently reducing the data transmission and computation needs. Its gate can (a) efficiently save the data transmission by early stopping the negative samples and (b) positively affect the model performance by (i) pre-training the early layers towards a better direction and (ii) simplifying the problem be removing negative samples early; Its \textit{Compression} layer can (a) efficiently save the data transmission by compressing on its output of the remaining samples for propagating across boundaries and (b) positively affect the model performance by discarding irrelevant or partially relevant dimensions. Together, the GC layer is able to reduce the amount of transmitted data efficiency by early stopping $96\%$ of negative samples and compressing $97\%$ per propagated sample, while improving accuracy by $1.06\sim6.12$ percentage points.

GC layers can be integrated into an existing network without modification. Then, the new network is able to be distributed across multiple heterogeneous compute islands to fully utilize all resources available. Therefore, larger and more powerful models can be built for better performance for Always-On use cases.

\newpage
\balance
\bibliography{paper}
\balance
\bibliographystyle{icml2023}
\nobalance
\clearpage
\appendix

\begin{center}
\Huge  \textsc{Appendix}
\end{center}
\section{Property of the Distributed Framework}

\begin{prop}\label{prop:distributed_framework}
The distributed framework will not affect the model prediction performance.
\end{prop} 

\begin{proof}
For any raw input $x$, the same output of $\mathcal{F}_\theta$ can be produced with all $k$ disjoint sub models together: 
\begin{equation*}
    \mathcal{F}_\theta(x)=\underbrace{\mathcal{F}^{k} \cdots\mathcal{F}^{1}}_\texttt{\tiny$k$ sub models}(x)=\mathcal{F}^{1 \mapsto k}(x).
\end{equation*}
\end{proof}

\section{Metrics of Gated}

To quantify and evaluate the performance of a Gate $\mathcal{G}$, a list of metrics are defined using True/False Positive/Negative ($TP$, $TN$, $FP$, and $FN$).

\theoremstyle{definition}
\begin{definition_}[\textbf{Stop Rate $\mathcal{P}_{sr}$}] \label{def:stop_rate} The percentage of samples, which are stopped by a Gate, is defined as:
\begin{equation*}
    \mathcal{P}_{sr}=\frac{TP +FP}{TP+FP+TN+FN}.
\end{equation*}
\end{definition_}

\begin{definition_}[\textbf{Negative Pass Through Rate $\mathcal{P}_{nptr}$}] The percentage of negative samples, which are mistakenly allowed to pass through by a Gate, is defined as:
\begin{equation*}
    \mathcal{P}_{nptr}=\frac{FP}{FP+TN}.
\end{equation*}
\end{definition_}

\begin{definition_}[\textbf{Positive Lost Rate $\mathcal{P}_{plr}$}] The percentage of positive samples, which are mistakenly stopped by a Gate, is defined as:
\begin{equation*}
    \mathcal{P}_{plr}=\frac{FN}{FN+TP}.
\end{equation*}
\end{definition_}

\begin{definition_}[\textbf{Negative Correction Rate $\mathcal{P}_{ncr}$}] The percentage of negative samples, which are correctly stopped by a Gate $\mathcal{G}$ but will be incorrectly classified later if the Gate lets them propagate through, is defined as:

\begin{equation*}
    \mathcal{P}_{ncr}=\frac{\left|\left\{(x,y)| \mathcal{G}(x)=0\wedge \mathcal{F}(x)\neq0\right\}\right|}{\left|\left\{(x,y)| y=0\right\}\right|}.
\end{equation*}
\end{definition_}

\begin{remark_}
A larger $\mathcal{P}_{ncr}$ is preferable since it decreases the difficulty for later sub models.
\end{remark_}

\section{Properties of \textit{Compression} layer}

\begin{prop}\label{prop:sparse_output}
The \textit{Compression} layer has the following property: a sparse weight matrix $\varphi$ leads to a sparse output $y$.
\end{prop}

\begin{cor}\label{cor:sparsity_regularization}
For a \textit{Compression} layer, encouraging the sparsity of $y$ is equivalent to putting a sparsity regularization on $\varphi$.
\end{cor}

\begin{definition_}[\textbf{Dropout Rate $\mathcal{P}_{dr}$}]\label{def:dropout_rate} The percentage of output dimensions, which are dropped out by a \textit{Compression} layer, is defined as:
\begin{equation*}
    \mathcal{P}_{dr}(\mathcal{C}_{\varphi})=1-\frac{|\varphi|_0}{\dim \varphi}.
\end{equation*}
\end{definition_}

\section{Binarizing the Weight Matrix}

To push most dimensions of the weight matrix $\varphi$ toward $\{0, 1\}$, we apply weight clipping $\sigma(\cdot)$ (Equation \ref{weight_clipping}) to limit the domain. Specifically, we use the element-wised ReLU-1 \citep{relu-6} activation function for constraining all dimensions into the range of $[0, 1]$.
\begin{equation}\label{weight_clipping}
    \sigma(\varphi) = \text{min}(1, \text{max}(0, \varphi))
\end{equation}

During the forward propagation, the binarize function, defined in Equation \ref{binarize}, is used for converting a floating weight into a binary weight.
\begin{equation}\label{binarize}
    \Gamma(\varphi) = \mathbf{1}_{{\varphi}>\frac{1}{2}}, \forall \varphi \in[0, 1]
\end{equation}

\begin{prop}
The \textit{Compression} layer with binarized weight matrix has the following property: the sparsity of $y$ is controlled by the sparsity of $\Gamma(\varphi)$ instead of $\varphi$ itself.
\end{prop}

\begin{cor}
To yield sparse $\Gamma(\varphi)$ and $y$,  $\mathcal{L}_\text{trans}$ (in Equation \ref{compression_loss}) could be $\mathcal{L}_p, \forall p>0$. 
\end{cor}

Since $\mathcal{L}_2$ is smoother than $\mathcal{L}_1$ in the domain $\varphi\in[0,1]$, which yields better degree of control. Therefore, in this paper, we use $\mathcal{L}_2$ for the data transmission regularization. Figure \ref{fig:w_hist} is an real example of weight distribution using $\mathcal{L}_2$ regularization.

\begin{conjecture_}\label{conj:sparsity}
Without regularization, the \textit{Compression} layer with binarized weight matrix still encourages sparse output.
\end{conjecture_}

\begin{proof}
During the forward propagation, all weight dimensions are binarized into $\{0, 1\}$, therefore, the dimensions of $0$s will yield $0$s in the corresponding dimensions of the output (Proposition \ref{prop:sparse_output}).

Assuming all weight dimensions are drew from a Bernoulli distribution: $\varphi_i\sim\text{Bernoulli} (p), p\in(0,1), \forall i \in[0,\dim \varphi]$.  Then, the expected sparsity is $(1-p)$, which is lager than $0$.

Additionally, our experiment results also confirm that there are about 40\% sparsity in the activation outputs even without regularization. 
\end{proof}

As shown in Equation \ref{binarize} and Figure \ref{fig:binarize}, the derivative of $\Gamma(\cdot)$ is: $\frac{\mathrm{d} \Gamma}{\mathrm{d} \varphi}=0, \forall \varphi \in[0, 1]$.
Therefore, it is impossible to do gradient back-propagating. Instead, the straight through estimator \citep{EstimatingGradients,courbariaux2016binarized, ste} is applied during back-propagation for estimating the gradient. 

\begin{figure}[t]
\includegraphics[width=0.485\textwidth]{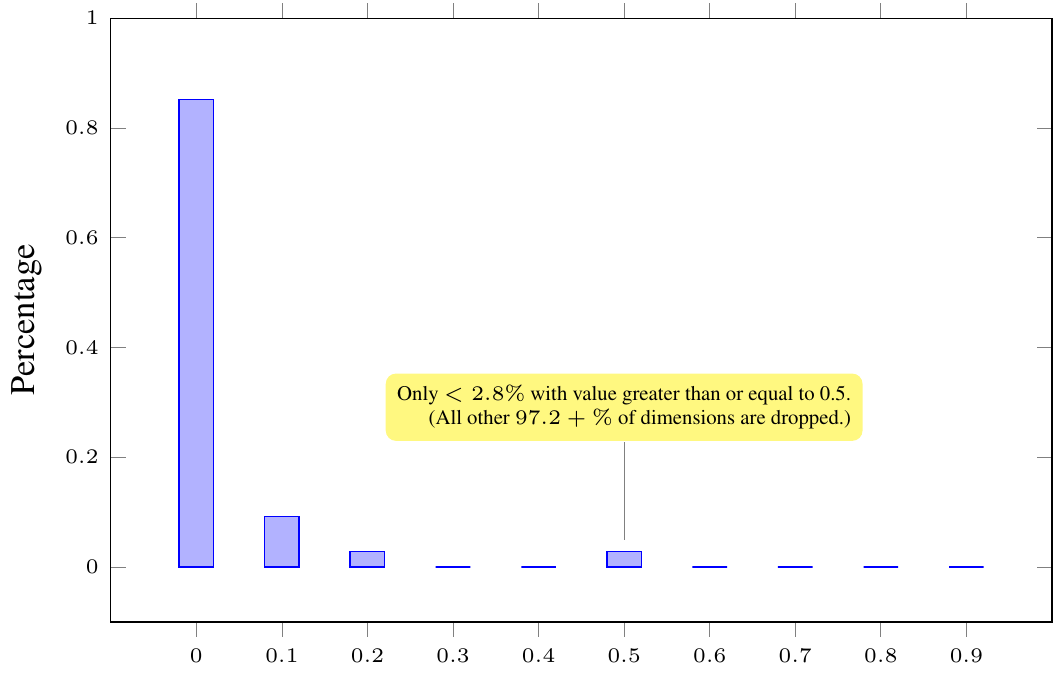}\caption{A real example of the weight distribution after 100 epochs on the Cifar10 dataset.
		}\label{fig:w_hist}
	
\end{figure}

\section{Compactness of the \textit{Compression} layer}

\begin{prop}
The \textit{Compression} layer has the following property: By binarizing, the data amount of the weight matrix \footnote{The weight matrix is stored in floating numbers when training, but all dimensions are binarized with $\Gamma(\cdot)$ (Equation \ref{binarize}) during the forward propagation. Thus, binarizing them into $\{0,1\}$s for deploying will not affect the model performance.} 
can be reduced by $N$ (number of bits required for one weight dimension in the default datatype) times for deploying.
\end{prop}

\begin{proof}
Once binarized, any dimension in $w$ can be represented with 1-bit ($\{0,1\}$) instead of a $N$-bits datatype (for example, in Tensorflow, the default datatype is float32, which is 32-bits).
\end{proof}

Additionally, the $\mathcal{L}_\text{trans}$ regularization (Equation \ref{compression_loss}) encourages sparse weight matrix, for which sparse encoding can reduce the size further. 
\begin{cor}
Let $m$ be the number of dimensions, $p$ be the percentage of the non-zero weights, and $N$ is the number of bits required for a weight in the default datatype, then the data compression rate of sparse encoding is 
\begin{equation}\label{data_compressed_rate}
   \mathcal{P}_{cr}(m,p)=\frac{N}{p\left \lceil \log_2(m) \right \rceil}.
\end{equation}
\end{cor}

\begin{proof}
The number of bits required to encode the index of all $m$ is $\left \lceil \log_2(m) \right \rceil$, and there are $pm$ non-zero dimensions to be encoded. Therefore, the number of bits required for encoding all non-zero dimensions is $pm\left \lceil \log_2(m) \right \rceil$.

Without encoding, each weight is stored in $N$-bits, which in total equates to $Nm$ bits; With sparse encoding, the required bits are reduced from $Nm$ to $pm\left \lceil \log_2(m) \right \rceil$ , which yields a compression rate of $\frac{N}{p\left \lceil \log_2(m) \right \rceil}$.
\end{proof}

Normally, $p\leq5\%$, $m\leq2^{9}$ and $N=32$ from our experiment results. With Equation \ref{data_compressed_rate}, we have $\mathcal{P}_{cr}(5\%,2^9)={71}$. In total, its parameter size is ${32*2^9}/{71}=230.8$ bits (or ${2^9}/{71}=\textbf{7.2}$ float32), which is minimal and compact enough for edge hardware modules\footnote{Normally, the edge hardware for ML has 1+k RAM and 10+k ROM. For example, ARM Cortex-M0 has 8k RAM (but needs to hold FW memory as well) and 64k ROM.}.
\section{Model Architecture}

\begin{figure}
\centering\includegraphics[width=0.485\textwidth]{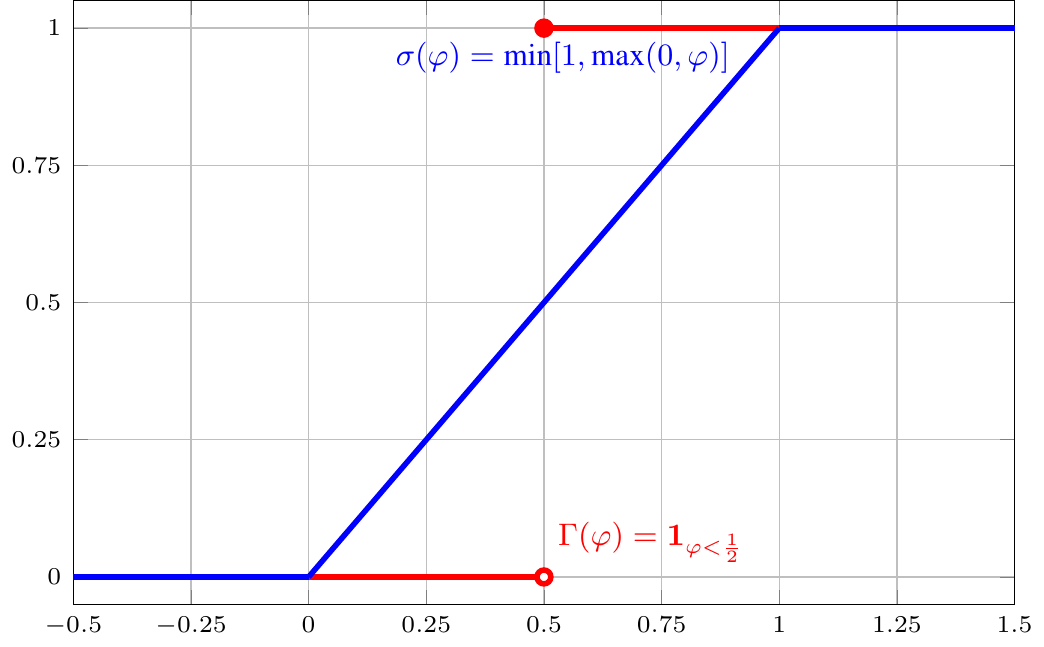}\caption{Plot of the Binarize Function $\Gamma(\cdot)$ and the Weight Clipping Function $\sigma(\cdot)$}
	\label{fig:binarize}
\end{figure}
\begin{figure}[H]
	\centering\includegraphics[width=0.495\textwidth]{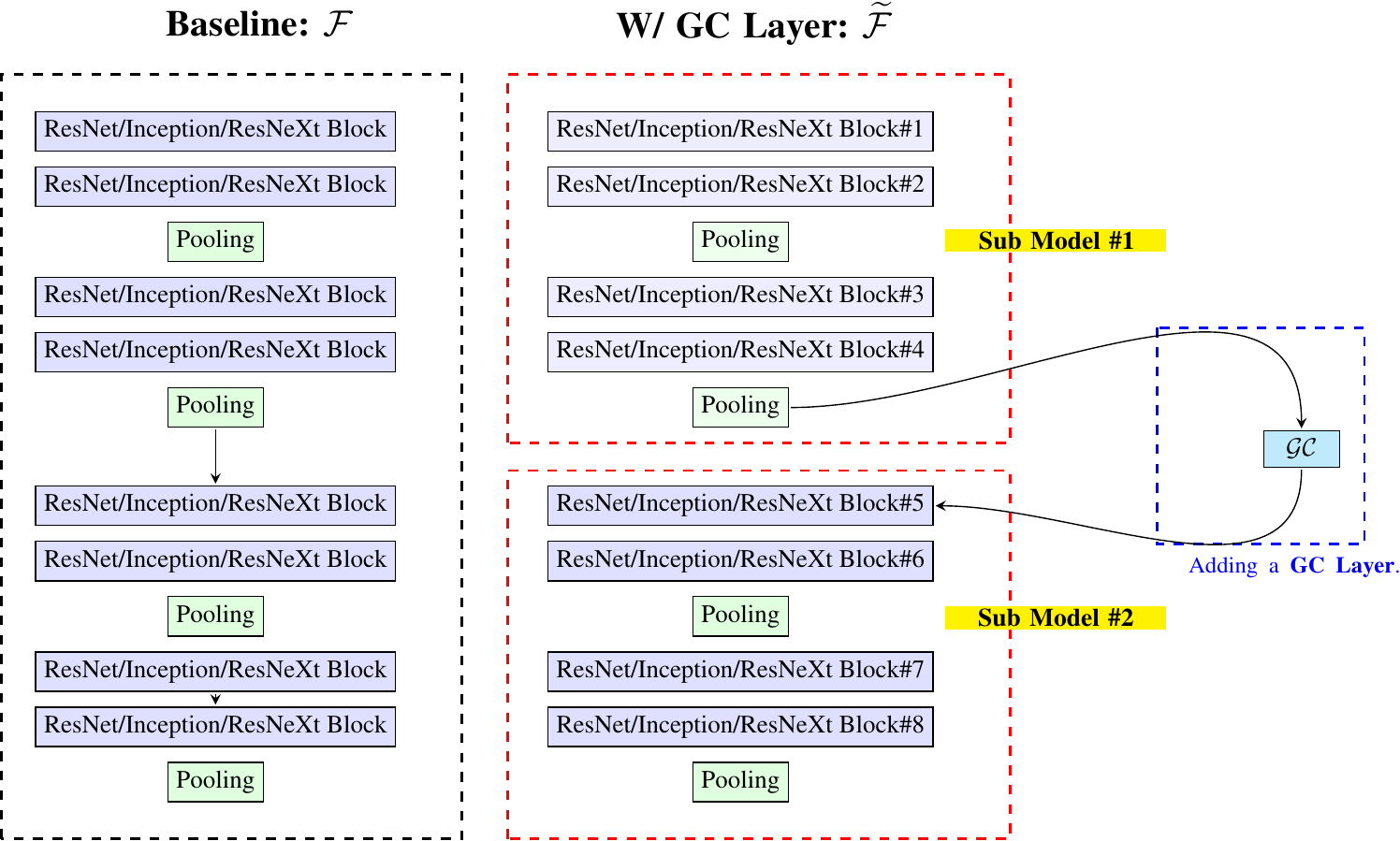}\caption{Model Architectures: the Baseline $\mathcal{F}$ model and the new $\widetilde{\mathcal{F}}$ model built from Baseline by adding one GC Layer. Note: (1) The model architectures consist of 8 ResNet/Inception/ResNeXt blocks, and additional 2 linear blocks for the classification head; (2) A GC layer can be inserted after each of these blocks.
		}\label{fig:model_arch}
	\end{figure}

\section{Deeper Analysis}

A set of experiments on the Cifar10 dataset are performed to analyze the properties of the proposed GC layer deeper. 

\subsection{Effectiveness of $\mathcal{G}$}\label{gate_eff}
To evaluate the performance of the \textit{Gate} $\mathcal{G}$, a set of experiments are conducted by varying its hyperparameter $\alpha$. To separate its impact from the \textit{Compression} layer $\mathcal{C}$, the weight matrix of $\mathcal{C}$ is set to all 1s to deactivate it.

\textit{$\mathcal{G}$ Positively Affects the Model Performance.} In Figure \ref{fig:effect_of_alpha_on_model_performance}, for $\alpha\in[0.05,0.8]$, $\widetilde{\mathcal{F}}$ consistently performs better than Baseline with, which indicates that $\mathcal{G}$ can boost the performance of $\widetilde{\mathcal{F}}$. This may be contributed by 1) similar to pre-training, it may tune the early layers to a better direction; 2) it reduces false positives by stopping negative samples earlier.


\textit{Larger $\alpha$ Encourages Better $\mathcal{G}$.} Figure \ref{fig:effect_of_alpha_on_gate_performance} shows both positive lost rate ($\mathcal{P}_{plr}$) and negative pass through rate ($\mathcal{P}_{nptr}$) decrease along with the increase of $\alpha$. 
This aligns with our expectations as a larger $\alpha$ places more weight on $\mathcal{G}$ during training. 

\textit{$\mathcal{G}$ Reduces Difficulty for $\widetilde{\mathcal{F}}$.} Figure \ref{fig:effect_of_alpha_on_gate_performance} shows that the negative corrected rate ($\mathcal{P}_{ncr}$) is always above 0, indicating that $\mathcal{G}$ decreases the complexity for $\widetilde{\mathcal{F}}$.

\begin{figure}[t!]
	\centering\includegraphics[width=0.485\textwidth]{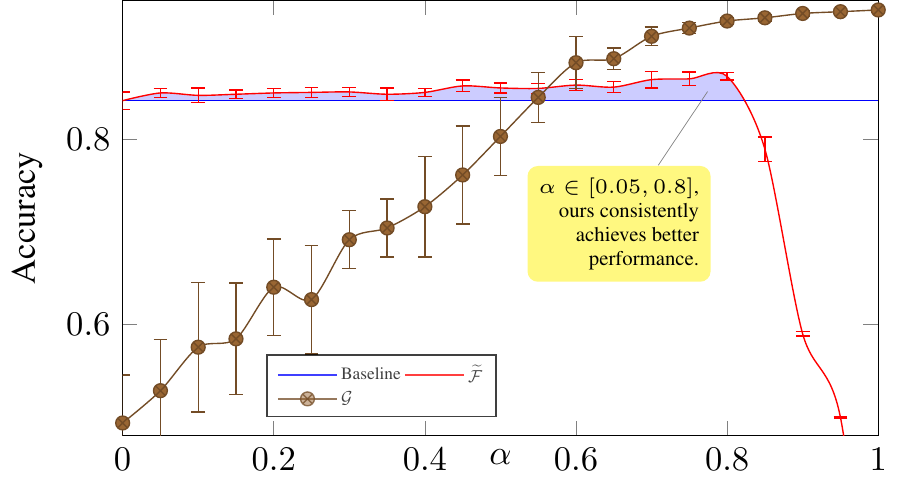}\caption{ Effect of Changing $\alpha$ on $\widetilde{\mathcal{F}}$. 
	}\label{fig:effect_of_alpha_on_model_performance}
	
\end{figure}
\subsection{Gating Analysis}
A gate can have a good early stopping performance for the negative samples, while incorrectly stopping a good percentage of positive samples at the same time.
Therefore, it is important to analyze the gating performance to ensure that the majority of positive samples pass through the network end-to-end for the final classification or prediction task.

The results in Figure \ref{fig:gate_roc} show our GC models consistently outperform BranchyNet in gating performance, as shown by the AUC and ROC curves. 

\begin{figure}[t]
	\centering\includegraphics[width=0.485\textwidth]{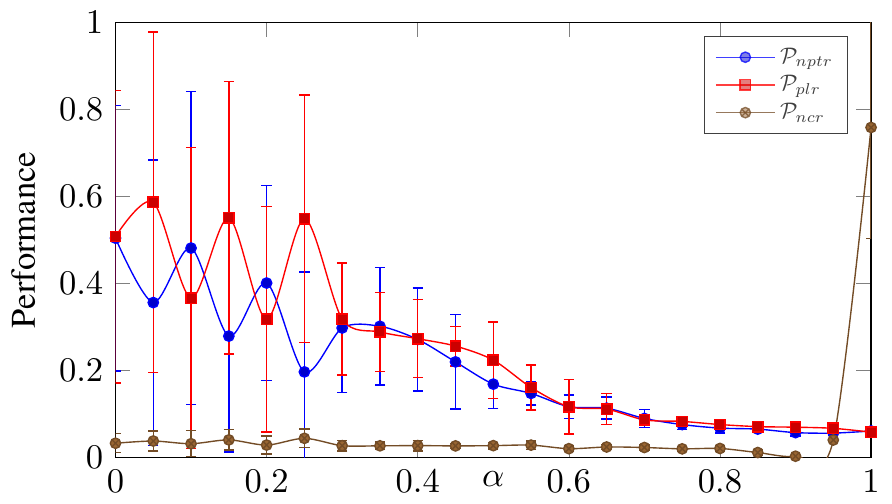}\caption{ Effect of Changing $\alpha$ on $\mathcal{G}$. 
	}\label{fig:effect_of_alpha_on_gate_performance}
\end{figure}

\begin{figure*}[tb]
        \centering
        \includegraphics[width = 0.995\textwidth]{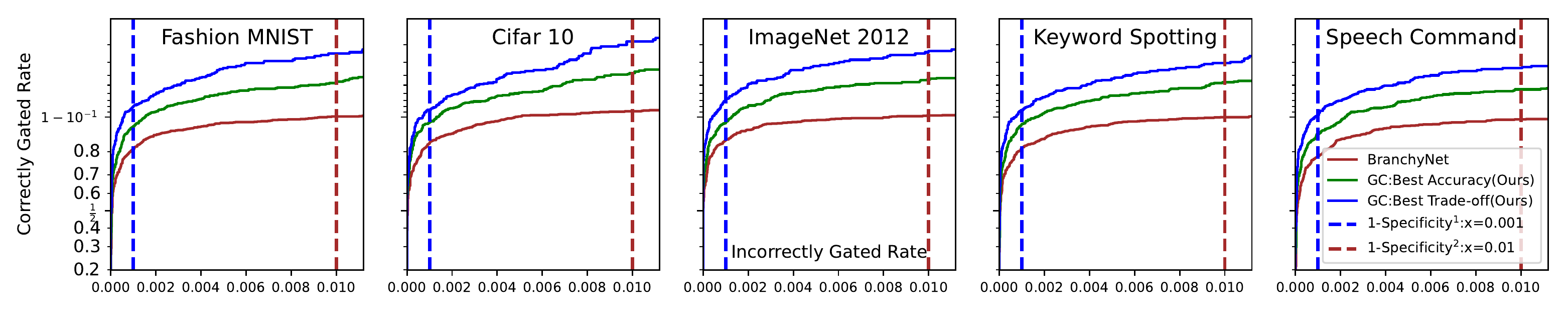}
        \caption{ROC Curves for Early Stopping Across All Five Datasets.}\label{fig:gate_roc}
\end{figure*}

\subsection{Effectiveness of $\mathcal{C}$}\label{dropout_eff}
To understand the effectiveness of the \textit{Compression} layer $\mathcal{C}$, a set of experiments are carried out by changing its hyperparameter $\beta$. To isolate it from the Gate $\mathcal{G}$, $\alpha$ is fixed at $0.5$. The results are reported in Figure \ref{fig:changing_beta}.

\textit{$\mathcal{C}$ Encourages Sparsity Even without Regularization.} When $\beta=0$, it still achieves the density of 0.579. This empirically confirms the Conjecture \ref{conj:sparsity}.

\textit{$\mathcal{C}$ Efficiently Controls Activation Sparsity with $\beta$.} A larger $\beta$ leads to a more sparse output. When $\beta=0.72$, it achieves density of 2.7\% while maintaining the accuracy of 0.879. In another words, it allows for a 97.3\% reduction in data transmission without compromising accuracy. 

\textit{$\mathcal{C}$ Positively Affects the Model Performance.} Similar to feature selection and dimensional reduction, the \textit{Compression} layer improves and stabilizes the model performance by dropping irrelevant or partially relevant dimensions to avoid their negative impact on the model performance. When $\beta<0.72$, $\widetilde{\mathcal{F}}$ achieves $1\%\sim 4.5\%$ accuracy gain comparing to Baseline with the same $\alpha=0.5$.

\begin{figure}[tb!]
	\centering\includegraphics[width=0.485\textwidth]{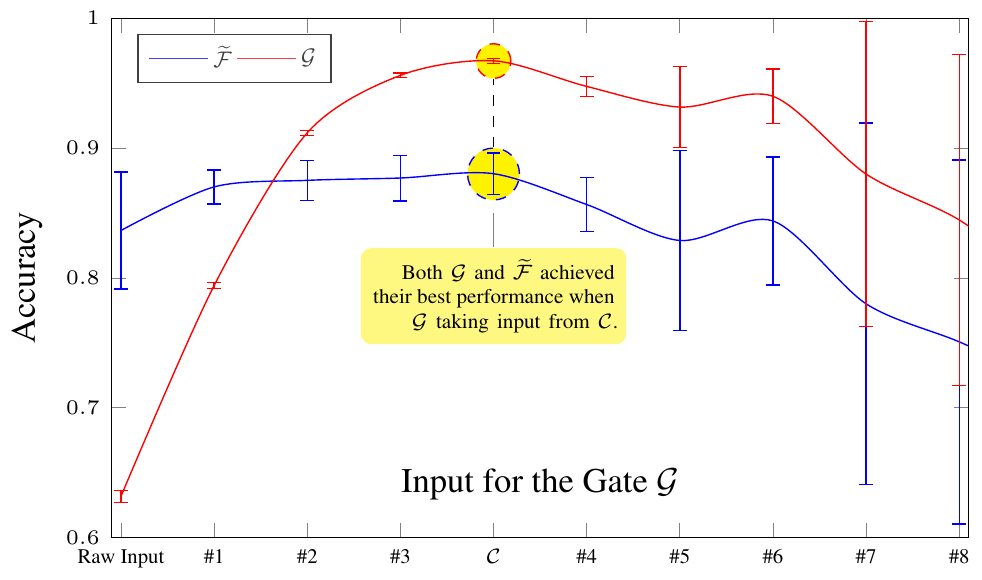}\caption{ Effect of Changing $\beta$ on $\widetilde{\mathcal{F}}$. 
	}\label{fig:changing_beta}
\end{figure}

\subsection{Inputs for $\mathcal{G}$}

From Figure \ref{fig:model_arch}, there are 8 ResNet/Inception/ResNeXt blocks, to understand the effect of the input layer of the Gate $\mathcal{G}$, a set of experiments are performed by linking its input to different layers. Based on the results in Table \ref{tab:overall}, the hyperparameters are chosen as $\alpha=0.7, \beta=0.6$. The results are reported in Figure \ref{fig:gate_inputs}.

\usetikzlibrary{patterns}
\usetikzlibrary{shapes.geometric,shapes,snakes,chains,arrows}
\usetikzlibrary{trees} %
\usetikzlibrary{spy}

\makeatother
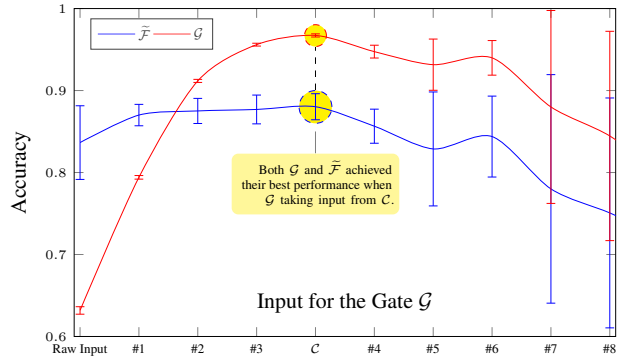
\begin{figure}[tb!]
\centering	\resizebox{0.485\textwidth}{!}{
\begin{tikzpicture}
	\tikzset{
		every pin/.style={fill=yellow!50!white,rectangle,rounded corners=3pt,font=\tiny},
		small dot/.style={fill=black,circle,scale=0.3}
	}
\begin{axis}[
xlabel={Input for the Gate $\mathcal{G}$},
x label style={at={(0.5,0.155)},anchor=north},
ylabel={Accuracy},
xticklabel style={rotate=0, anchor=near xticklabel},
xtick=data,
tick label style={font=\tiny},
xticklabels={\#9, \#8, \#7, \#6, \#5, \#4,$\mathcal{C}$, \#3, \#2, \#1, Raw Input},
width=0.6\textwidth, 
height=0.4\textwidth,
ymin=0.6,
ymax=1.,
xmin=-0.1,
xmax=9.1,
legend cell align={left},
legend columns=4,
legend pos=north west,
legend style={draw opacity=.75,text opacity=.75,fill opacity=0.5,nodes={scale=0.65, transform shape}},
]

\node[circle,draw=red,fill=yellow,inner sep=1pt,minimum size=10pt,dashed] (a) at (axis cs:4,0.967) {};
\node[circle,draw=blue, fill=yellow,inner sep=1pt,minimum size=15pt, dashed] (b) at (axis cs:4,0.88) {};

\addplot+ [
    smooth,
    mark=none,
    error bars/.cd, 
    y fixed,
    y dir=both, 
    y explicit
    ] table [x = x, y = y, y error = ey, col sep = comma] {gate_input.csv};
\addlegendentry{$\widetilde{\mathcal{F}}$}

\addplot+ [
    smooth,
    mark=none,
    error bars/.cd, 
    y fixed,
    y dir=both, 
    y explicit
    ] table [x = x, y = g, y error = eg, col sep = comma] {gate_input.csv};
\addlegendentry{$\mathcal{G}$}

\draw [-,dashed] (b.north) -- (a.south);
\node[pin={[align=right,, text width=7em, pin distance=0.35cm]-90:{\tiny{Both $\mathcal{G}$ and $\widetilde{\mathcal{F}}$ achieved their best performance when $\mathcal{G}$} taking input from $\mathcal{C}$.}}] at (b.south) {};
\end{axis}

\end{tikzpicture}}\caption{Effect of Changing $\mathcal{G}$'s Input.
}\label{fig:gate_inputs}
\end{figure}

\textit{The Performance of $\mathcal{G}$ Increases Significantly along with the Movement of Placing $\mathcal{G}$ Closer to $\mathcal{C}$.} This is expected as there are more layers to be tuned for better performance. 

\textit{The Performance of $\widetilde{\mathcal{F}}$ Increases along with the Movement of Placing $\mathcal{G}$ Closer to $\mathcal{C}$.} This aligns with our previous observation in relation to performance: a better $\mathcal{G}$ can also benefit $\widetilde{\mathcal{F}}$ to achieve better performance. 


Additionally, placing the gate before $\mathcal{C}$ requires a larger $\mathcal{G}$ as the output of an early layer without compression or dropping tends to be larger. A possible workaround is adding additional pooling layers (for example, AvgPool, MaxPool, Conv1x1) to shrink the input. However, this gradually disengages $\mathcal{G}$ from $\widetilde{\mathcal{F}}$. Therefore, it decreases the pre-training benefit for $\widetilde{\mathcal{F}}$.

\textit{The Performances of Both $\widetilde{\mathcal{F}}$ and $\mathcal{G}$ Decrease Once $\mathcal{G}$ Is After $\mathcal{C}$.} The reasons are: (1) the output of the later blocks after $\mathcal{GC}$ is designed to have way more channels that makes the input more noisy for $\mathcal{G}$; (2) the overlapping of $\widetilde{\mathcal{F}}$ and $\mathcal{G}$ is large, which leads to greater competition than cooperation.

\begin{figure}[tb!]
	\centering\includegraphics[width=0.485\textwidth]{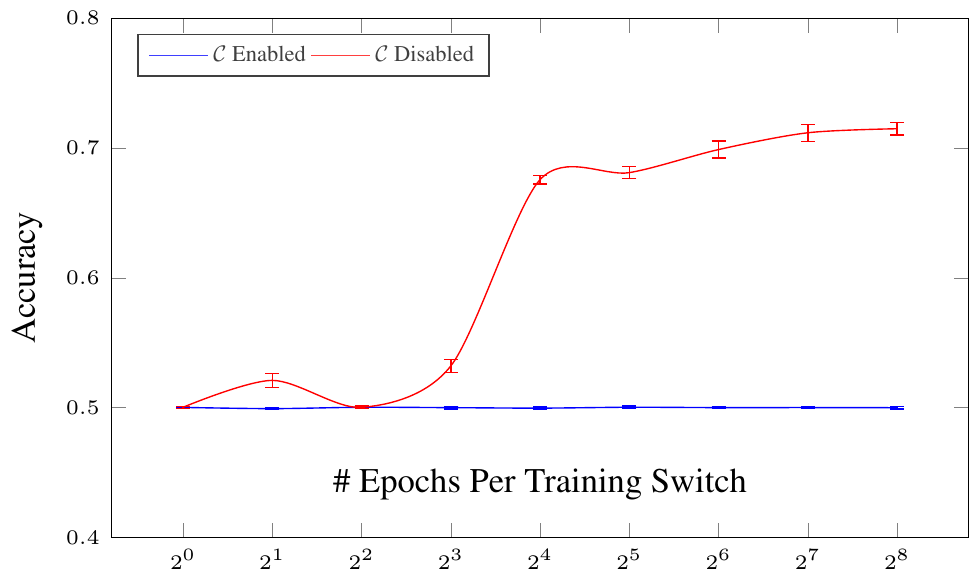}\vspace{-8pt}\caption{Two Stages Training Performances.}\label{fig:two_stages}
\end{figure}
Overall, it is preferable to connect the input layer of $\mathcal{G}$ to the output from $\mathcal{C}$ since it generates the best performance for both $\widetilde{\mathcal{F}}$ and $\mathcal{G}$. Additionally, this also streamlines  the implementation of the GC layer, as the connection is internal.

\subsection{Manually Two Stages Training}
In this section, we are exploring the difference between training $\mathcal{G}$ and $\widetilde{\mathcal{F}}$ end-to-end simultaneously versus manually training them in two stages and then merging together for inference. To simulate the two stages training schema, the gradient flow between the two sub-models is intentionally halted. Subsequently, we alternate between training $\mathcal{G}$ for $N$ epochs and $\widetilde{\mathcal{F}}$ for $N$ epochs, repeating this process until a total of 512 epochs are reached. When the number of epochs per training switch is $2^8$ (256), it forms a hierarchical ensemble model with two sub-models trained in sequence. The results are reported in Figure \ref{fig:two_stages}.

\textit{Disabling $\mathcal{C}$ Performs Slightly Better Than Enabling $\mathcal{C}$.} $\mathcal{C}$ is added on purpose to drop less useful dimensions. Since the gradient is stopped between the two sub models, $\mathcal{C}$ is optimized for $\mathcal{G}$ only. The useful information for $\widetilde{\mathcal{F}}$ is further reduced when $\mathcal{C}$ is in effect. 
Moreover, the results suggest that simultaneously training all components end-to-end is more effective than a two-stage training approach, as the components can work together and optimize the overall performance.


\end{document}